\documentclass[letterpaper]{article} 
\usepackage[]{aaai23} 
\usepackage{times}  
\usepackage{helvet}  
\usepackage{courier}  
\usepackage[hyphens]{url}  
\usepackage{graphicx} 
\usepackage{natbib}  
\usepackage{caption} 
\usepackage{algorithm}
\usepackage{algorithmic}

\usepackage{newfloat}
\usepackage{listings}
\DeclareCaptionStyle{ruled}{labelfont=normalfont,labelsep=colon,strut=off} 
\floatstyle{ruled}
\newfloat{listing}{tb}{lst}{}
\floatname{listing}{Listing}
\usepackage{mathtools}
\newcommand{\defeq}{\vcentcolon=}

\newtheorem{theorem}{Theorem}[section]

\newtheorem{lemma}[theorem]{Lemma}

\newtheorem{conclusion}[theorem]{Conclusion}

\newtheorem{definition}[theorem]{Definition}

\newtheorem{proof}[theorem]{Proof}

\newtheorem{proposition}[theorem]{Proposition}

\usepackage{xcolor}

\def\rb{\right)}
\def\lb{\left(}

\def\rc{\right]}
\def\lc{\left[}

\def\rs{\right\}}
\def\ls{\left\{}

\def\sum#1#2#3{\overset{#1}{\underset{#2}{\Sigma}} #3}

\def\abs#1{\left| #1\right|}
\def\norm#1{\left|\left| #1\right|\right|}

\usepackage{mathtools} 
\usepackage{changepage} 

\usepackage{amssymb}

\newcommand{\equ}[1]{
    \begin{equation}
        #1
    \end{equation}
}

\usepackage{kantlipsum}
\allowdisplaybreaks

\title{
    Random Search Hyper-Parameter Tuning: Expected Improvement Estimation 
    And The Corresponding Lower Bound
}
\author {
    Dan Navon,\textsuperscript{\rm 1}
    Alex M. Bronstein\textsuperscript{\rm 1}
}
\affiliations {
    \textsuperscript{\rm 1} Department of Computer Science, Technion, Haifa, Israel\\
    dannav@technion.ac.il
}

\usepackage{bibentry}

\begin{document}

\maketitle

\begin{abstract}
Hyperparameter tuning is a common technique for improving the performance of neural networks. Most techniques for hyperparameter search involve an iterated process where the model is retrained at every iteration. However, the expected accuracy improvement from every additional search iteration, is still unknown. Calculating the expected improvement can help create stopping rules for hyperparameter tuning and allow for a wiser allocation of a project's computational budget. In this paper, we establish an empirical estimate for the expected accuracy improvement from an additional iteration of hyperparameter search. 
Our results hold for 
any hyperparameter tuning method which is based on random search 
\cite{bergstra2012random}
and samples hyperparameters from a fixed distribution.
We bound our estimate with an error of $O\lb\sqrt{\frac{\log k}{k}}\rb$ w.h.p. where $k$ is the current number of iterations. To the best of our knowledge this is the first bound on the expected gain from an additional iteration of hyperparameter search.
Finally, we demonstrate that the optimal estimate for the expected accuracy will still have an error of $\frac{1}{k}$.
\end{abstract}

\section{Introduction}
\paragraph{}
Hyperparameter tuning is a common technique for improving the accuracy of neural networks \cite{bischl2021hyperparameter}. 
There are various methods for performing this tuning 
\cite{yu2020hyper, yang2020hyperparameter, luo2016review}, 
in this work we focus on methods which iteratively sample hyperparameter configurations from a stable distribution 
\cite{huang2021novel, roman2021machine, callaghan2021machine}. 
These methods usually involve training the model on multiple sampled hyperparameters, 
in an attempt to find a hyperparameter configuration which improves the model's accuracy. 
This process is complex for a number of reasons:
\begin{itemize}
    \item Each hyperparameter sampling iteration, requires additional training and thus the number of iterations is limited by the researcher's computational budget.
    \item Throughout the tuning each iteration has a diminishing probability of improving accuracy.
    \item Rare rewards make the computational-budget/accuracy trade off hard to handle since it is unclear how to design an effective stopping rule. 
\end{itemize}

\paragraph{}
With no gold-standard for stopping rule different researchers evaluate their methods with different computational budgets, making it hard to compare their results. This has led to the creation of statistical estimators for the EVP - Expected Validation Performance \cite{dodge-etal-2019-show, tang-etal-2020-showing, dodge2021expected}, a tool used for reporting performance (e.g., accuracy) as a function of computational budget. EVP allows for comparison between researchers using different computational budgets by estimating the results of the researcher who had a higher computational budget if he had used a lower computational budget.

\paragraph{}
In this work we examine the opposite problem. We empirically estimate the expected gain in accuracy obtained by sampling one more hyperparameter configuration given the accuracies obtained from the models trained in the first $k$ iterations. To ensure the tightness of our estimation we show that the error is diminishing with $k$ with high probability. Additionally, we show that the estimation error cannot be better than $\frac{1}{k}$.


\subsection{Proof overview}
\paragraph{}
Let $X$ be a r.v. denoting the model accuracy for randomly sampled hyperparameters $h \sim \mathcal{H}$. Let $k$ be the number of hyperparameter search iterations preformed to date. Accurately estimating the accuracy gain from sampling the $k+1$ configuration of hyperparameters, and training the corresponding models, naively would require us to estimate the density function of $X$. However, empirical estimation of the density function $f_X$, is a hard problem and is currently not solvable without additional assumptions. In our case there is no prior knowledge about $f_X$ and we can make no assumptions regarding it.

\paragraph{}
We overcome this hurdle by estimating the empirical accumulated function $F_X$ instead. The estimation is obtained using the Dvoretzky–Kiefer–Wolfowitz Massart inequality (DKW) 
\cite{massart1990tight, naaman2021tight, bitouze1999dvoretzky}
from statistics that allows us to accurately estimate the empirical accumulated function $\widehat{F}_X$ in the $l_\infty$ norm. Using this estimation enables us to use the cumulative function $\widehat{F}_X$ for estimating the expected accuracy gain from the $k+1$ sample.  

Finally, we show that for normally distributed   
$X \sim \mathcal{N}\lb \mu, \sigma\rb$ the best estimator 
behaves like $I\lb \widehat{\mu},\, \widehat{\sigma}\rb$ 
for $\widehat{\mu},\, \widehat{\sigma}$ 
the empirical values of $\mu, \sigma$ and $I$ a convex function. 
We then bound the error of the 
best estimator from below by both 
\begin{align*}
    & 
    \lb 1\rb
    \quad
        \frac{\partial}{\partial \mu} I\lb \mu,\, \sigma\rb 
        \cdot \abs{\mu - \widehat{\mu}}\,
\end{align*}
and
\begin{align*}
    &
    \lb 2\rb
    \quad
        \frac{\partial}{\partial \sigma} I\lb \mu,\, \sigma\rb 
        \cdot \abs{\sigma - \widehat{\sigma}} 
\end{align*}
Taking the expectation of both expressions 
results in $\frac{1}{k}$ lower bound on the 
mean absolute error.

\section{Related work}

\paragraph{}
Hyperparameter tuning methods can be split into 4 major types: Grid search, Random search, Bayesian Optimisation, Meta Heuristic algorithms. We will discuss each one briefly.

\paragraph{}

Grid Search is a simple exhaustive method, its major limitation being that it is time-consuming and impacted by the curse of dimensionality \cite{claesen2014easy}. Thus, it is unsuitable for a large number of hyperparameters and is often imprecise for continuous parameters. It is also computationally expensive as it is an exhaustive search. Therefore Grid Search is only efficient for a small number of categorical hyper-parameters \cite{yang2020hyperparameter}.

\paragraph{}
Random Search is more efficient than Grid Search and supports all types of hyperparameters. In practical applications, using Random Search to evaluate
the randomly selected hyperparameter values helps explore a large search space \cite{yang2020hyperparameter}.

\paragraph{}
Bayesian optimization models include: BO-GP, SMAC, and BO-TPE - based on their surrogate models. To reduce unnecessary evaluations and improve efficiency theses models determine the next hyper-parameter value to be tested based on the previously evaluated results. They can be used for categorical, discre, continuous and conditional parameters \cite{yang2020hyperparameter}.

\paragraph{}
Metaheuristic algorithms, (such as GA and PSO), are more complicated tuning algorithms, yet they often perform well for complex optimization problems. They support all types of hyper-parameters and are particularly efficient for large configuration spaces, since they can obtain near-optimal solutions within very few iterations. These algorithms have their own disadvantages though as an appropriate initialization is crucial for the success of PSO and GA introduces its own hyperparameters, such as crossover and mutation rates~\cite{yang2020hyperparameter}.

A comprehensive comparison between the algorithms can be found in~\cite{yang2020hyperparameter}. Our work only pertains to algorithms based on Random Search, which despite having some disadvantages is still commonly used in many modern works \cite{huang2021novel, roman2021machine, callaghan2021machine}.


\section{Preliminaries}
\subsection{Dvoretzky–Kiefer–Wolfowitz inequality} 
\paragraph{}
For a fixed random variable $X$ the DKW inequality aims to bound the discrepancy between the empirical $X$ and the ground truth of the distribution. 
We will start by defining the empirical cumulative function. 

\begin{definition}
    Given a natural number $n$ let $X_1, X_2, \hdots, X_n$
    be real-valued independent and identically distributed
    random variables with cumulative distribution function 
    $F_X$. Let $F_n$ denote the associated empirical distribution function defined by
    \begin{equation*}
        F_n\lb x\rb 
        \,=\, 
        \frac{1}{n} \overset{n}{\underset{i=1}{\Sigma}}\, 
        \mathbf{1}_{X_i \leq x}, \quad\; x \in \mathbb{R}
    \end{equation*}
\end{definition}
In words $F_n(x)$ is the fraction of random variables that are smaller than $x$.
The Dvoretzky–Kiefer–Wolfowitz inequality \cite{massart1990tight, bitouze1999dvoretzky} bounds the probability that the random function $F_n$ differs from $F_X$ by more than a given constant $\epsilon > 0$  anywhere on the real line. Formally: 

\begin{theorem}\label{dkw_thm}
    Let $X_1, \hdots, X_n \sim \mathcal{D}$ be identically distributed random variables 
    and let $F_n$ denote the empirical cumulative distribution. 
    Then for every $0 < \alpha < 1$ the following inequlity bound holds 
    \begin{equation*}
        \norm{F_n\lb w\rb - F_X\lb w\rb}_{\infty} \leq \epsilon        
    \end{equation*}
    with probability $1 - \alpha\,\,$ for  
    \begin{equation*}
        \epsilon = \sqrt{\frac{\log \frac{2}{\alpha}}{2n}}        
    \end{equation*} 
\end{theorem}

\section{Estimator theorem}
\paragraph{}
In this section we will estimate the expected improvement by 
the $k+1$ iteration of hyperparameter sampling and will bound the error with high probability. 
Theorem~\ref{thm_11} presents a formal description of our main result 
and proceeds with the corresponding proof 

\begin{theorem}\label{thm_11}
    Let $h_1, \hdots, h_k \sim H$ be i.i.d sampled hyperparameters, 
    and let $X_1, \hdots, X_k$ be the model accuracies when 
    trained with hyperparameters $h_1, \hdots, h_k$ correspondingly. 
    Let $\widehat{F}_k$ be the empirical cumulative function
    induced by $X_1, \hdots, X_k$ 
    and let $\Delta_{k+1}$ be the random variable denoting 
    the expected accuracy gain from sampling the $k+1$ hyperparameter. 
    The expected accuracy gain can be estimated by 
    \equ{
        \underset{h_{k+1}\sim \mathcal{D}_\mathcal{H}}{\mathbb{E}}\lc \,\Delta_{k+1}\,\rc 
        = \int_0^1\, \lb \widehat{F}_k \rb^{\,k}\,\cdot\,\lb 1-\widehat{F}_k\rb \, d t \,+\, \epsilon_\textit{error} 
    } 
    with probability of at least $1-\frac{1}{\sqrt{k}}\,$ for 
    \equ{
        \epsilon_\textit{error} \leq 
        6\sqrt{\frac{\log k}{k}} 
    }   
\end{theorem}

\section{Theorem development}
\paragraph{}
Let $\ls h_k\rs_{k \in \mathbb{N}} \subset H$ be i.i.d randomly sampled 
hyperparameters and let $\ls X_k\rs_{k \in \mathbb{N}}$ 
be the random variable representing the model accuracy 
at time $k$ when trained with $h_k$. 
Denoting by $Y_k$ the best model accuracy up to time $k$,  
then 
\begin{align*}
        &
        Y_k = \underset{i \leq k}{\max}\lb \ls X_i\rs_{i \leq k}\rb 
\end{align*}
and 
\begin{align*}
        &
        \Delta_k 
        \,=\,
        Y_{k} - Y_{k-1} 
\end{align*}

We will start by estimating
$\;Y_{k},\; Y_{k} - Y_{k-1}\;$ 
as functions of 
$\;X_1,...,X_k,\; Y_1,...,Y_{k-1}\;$ 
then observe that
\begin{equation*}
        F_{Y_k}\lb w\rb
        = P \lb Y_k \leq w\rb
        = P \lb X_1,...,X_k \leq w\rb
\end{equation*}
\begin{equation*}
        = \underset{i \leq k}{\prod} P \lb X_i \leq w\rb
        = \underset{i \leq k}{\prod} F_{X_i}\lb w\rb
        = F_{X}\lb w\rb^{k}
\end{equation*}
Thus
\begin{equation*}
    F_{Y_k}\lb w\rb = F_{X}\lb w\rb^{k}     
\end{equation*}
hence 
\begin{equation*}
    \mathbf{E} \left[Y_k\right] = \int_{0}^1 1-F_{Y_k}\lb w\rb\; dw = 1 - \int_{0}^1 F_{X}\lb w\rb^k\; dw
\end{equation*}
and 
\begin{align*}
    & 
       \Delta_{k+1} \equiv \mathbf{E} \left[Y_{k+1} - Y_k\right]
    \\ & 
        = \lb 1 - \int_{0}^1 F_{X}\lb w\rb^{k+1}\; dw \rb - \lb 1 - \int_{0}^1 F_X\lb w\rb^k\; dw \rb         
    \\ & 
        = \int_{0}^1 F_X\lb w\rb^k - F_X\lb w\rb^{k+1}\; dw
    \\ & 
        = \int_{0}^1 F_X\lb w\rb^k \lb 1 - F_X\lb w\rb\rb\; dw
\end{align*}
and by the DKW theorem (\ref{dkw_thm}) we have  
\begin{equation*}
    \norm{\widehat{F}_k  \lb w\rb - F_{X}\lb w\rb}_{\infty} \leq \epsilon        
\end{equation*}
with probability $1 - \alpha\,$ 
for 
\begin{equation}\label{eq_eps_choice}
    \epsilon = \sqrt{\frac{\log \frac{2}{\alpha}}{2k}}   
\end{equation}
Thus denoting by $G\,$ the event that the DKW approximation holds then 
\begin{align*}
    & 
    \Delta_{k+1} 
    \equiv \mathbf{E} \left[Y_{k+1} - Y_k\right]
    = \mathbf{E} \left[\lb Y_{k+1} - Y_k\rb \lb \mathbf{1}_{G} + \mathbf{1}_{G^c}\rb\right]             
    \\ & 
    \leq \alpha + \lb 1-\alpha\rb \mathbf{E}\left[Y_{k+1} - Y_k| G\right]
    \leq
    \alpha + \mathbf{E}\left[Y_{k+1} - Y_k| G\right]            
\end{align*}
similarly 
\begin{align*}
    & 
        \Delta_{k+1} 
        \,\geq\, \lb 1-\alpha\rb \cdot \mathbf{E}\left[Y_{k+1} - Y_k| G\right]
\end{align*}

Given $G$ we may estimate $\mathbf{E}\left[Y_{k+1} - Y_k| G\right]$ 
by using the DKW theorem. for replacing
$F_X$ and  $\widehat{F}_k$
\begin{align*}
    & 
        \mathbf{E}\left[Y_{k+1} - Y_k| G\right] 
    \geq
        \int_0^1 \lb \widehat{F}_k - \epsilon\rb^k
        \lb 1 - \epsilon - \widehat{F}_k\rb\, dw
    \equiv 
        LB_{k+1} 
\end{align*}
similarly 
\begin{align*}
    & 
            \mathbf{E}\left[Y_{k+1} - Y_k| G\right] 
        \leq
            \int_0^1 \lb \widehat{F}_k +\epsilon\rb^k
            \lb 1 + \epsilon - \widehat{F}_k\rb\, dw
        \equiv 
            UB_{k+1}        
\end{align*}
combining the last three equations 
\begin{align*}
    \Delta_{k+1} 
    \,=\,
    \int_0^1 
    \lb \widehat{F}_k\rb^k
    \lb 1 - \widehat{F}_k\rb
    \,dw \,+\, \epsilon_{err}    
\end{align*}
when $\epsilon_{err}\,$ is bounded by  
\begin{align*}
    & 
    \epsilon_{err} 
    \,\leq\, 
    UB_{k+1} \,-\, LB_{k+1}   
     \,+\, 
    2 \alpha
\end{align*}
choosing $\alpha=\frac{1}{\sqrt{2 k}}\,$ 
and using the lemma (\ref{lem_764}) 
to bound $UB_{k+1} \,-\, LB_{k+1}$
we have 
\begin{align*}
    & 
        \epsilon_{err} 
        \,\leq\, 
        UB_{k+1} \,-\, LB_{k+1}   
         \,+\, 
        2 \alpha
    \\ & 
        \,\leq\, 
            2\, \alpha + \lb 2+\frac{2}{e}\rb\, \epsilon 
        \,\leq\, 
            3\lb \alpha + \epsilon \rb  
    \\ & 
        \underbrace{\,=\,}_{\lb 3\rb} 
            3\lb \alpha + \sqrt{\frac{\log \frac{2}{\alpha}}{2k}} \rb  
        \underbrace{\,=\,}_{\alpha = \frac{1}{\sqrt{2 k}}} 
           3\lb \frac{1}{\sqrt{2 k}} + \sqrt{\frac{\log\, \sqrt{2 k}}{2k}} \rb
    \\ & 
        \,=\, 
            3\lb \frac{1}{\sqrt{2 k}} + \sqrt{\frac{\log\, 2 k}{k}} \rb
        \,\leq\, 
            6 \, \sqrt{\frac{\log\, k}{k}} 
\end{align*}
when transition $\lb 3\rb$ is from $\epsilon$ choice (\ref{eq_eps_choice})
$\hfill\square$

Thus it is left to prove the following lemma for us to end the theorem 
development (\ref{thm_11})
 
 \begin{lemma}\label{lem_764}
     Let $k \in \mathbb{N}$ be an integer and let $\epsilon > 0$ 
     denote by 
     \begin{align*}
         &
        LB_{k+1}
        \,=\,
        \int_0^1 \lb \widehat{F}_k - \epsilon\rb^k
        \lb 1 - \epsilon - \widehat{F}_k\rb\, dw
        \\ & 
         UB_{k+1}
         \,=\,
        \int_0^1 \lb \widehat{F}_k +\epsilon\rb^k
        \lb 1 + \epsilon - \widehat{F}_k\rb\, dw
     \end{align*}
    then 
    \begin{align*}
        UB_{k+1} - LB_{k+1} 
        \,\leq\,  
        \lb 2+\frac{2}{e}\rb \epsilon
    \end{align*}
 \end{lemma}
 \begin{proof}
    And indeed using the $UB_{k+1}$, $LB_{k+1}$ definitions 
    \begin{align*}
        &
            UB_{k+1}-LB_{k+1}
        \\ & 
            \quad\quad\quad\quad
            = \int_0^1 \lb 1-\widehat{F}_k\rb \lc \lb \widehat{F}_k+\epsilon\rb^k
            - \lb \widehat{F}_k-\epsilon\rb^k \rc\; dw
        \\ & 
            \quad\quad\quad\quad\quad\quad\quad
            \,+\, 
            \epsilon \int_0^1 \lc \lb \widehat{F}_k+\epsilon\rb^k + \lb \widehat{F}_k - \epsilon\rb^k \rc\; dw
        \\ & 
        \leq \int_0^1 \lb 1-\widehat{F}_k\rb \lc \lb \widehat{F}_k + \epsilon\rb^k - \lb \widehat{F}_k - \epsilon\rb^k \rc\; dw
        + 2\epsilon    
    \end{align*}
    denote by 
    \begin{align*}
        I \,=\, \int_0^1 \lb 1-\widehat{F}_k\rb \lc \lb \widehat{F}_k + \epsilon\rb^k - \lb \widehat{F}_k - \epsilon\rb^k \rc\; dw
    \end{align*}
    and observe that 
    \begin{align*}
        & 
            \lb \widehat{F_D}+\epsilon\rb^k - \lb \widehat{F_D}-\epsilon\rb^k 
        \\ & \quad\quad\quad\;\;
        \,=\, 
            2\epsilon\lc
            \frac{\lb \widehat{F_D}+\epsilon\rb^k - \lb \widehat{F_D}-\epsilon\rb^k}{2\epsilon} 
            \rc
        \,\approx\, 2\epsilon\,k\,\widehat{F_D}^{k-1}    
    \end{align*} 
    thus 
    \begin{align*}
        &
        I 
        \,=\, 
        \int_0^1 \lb 1-F_k\rb \lc \lb \widehat{F}_k \,+\, \epsilon\rb^k - \lb \widehat{F}_k \,-\, \epsilon\rb^k \rc\; dw
        \\ &
        \,\approx\, 
        2\,\epsilon\,k\,\int_0^1 \lb 1-\widehat{F}_k\rb\cdot\lb \widehat{F}_k\rb^{k-1}\; dw
        \\ &
        \,=\, 
        2\epsilon\,k \lc\underset{\widehat{F}_k \leq 1-\delta}{\int}+\underset{1-\delta\leq \widehat{F}_k}{\int}\rc 
        \lb 1-\widehat{F}_k\rb\cdot\lb \widehat{F}_k\rb^{k-1}\; dw
        \\&
        \,\leq\, 
        2\epsilon\,k 
        \norm{\lb 1-\widehat{F}_k\rb\cdot\lb \widehat{F}_k\rb^{k-1}\;}_\infty
        \\ & 
        \,=\,
        2\,\epsilon\,k \underset{x \in \lc 0, 1\rc}{\max}\, \lb 1-x\rb\, x^{k-1}    
        \\ & 
        \underbrace{\,=\,}_{\lb 6\rb} 
        2\epsilon\,k \lb \frac{\frac{1}{k}}{\frac{k-1}{k}}\rb \lb 1-\frac{1}{k}\rb^{k}
        \\ & 
        \,\approx\, 
        2\epsilon\,k \lb \frac{1}{k-1}\rb e^{-1}
        \,\approx\,
        \frac{2}{e} \epsilon
    \end{align*}
    when transition $\lb 6\rb$ is because the function 
    \begin{equation*}
        f\lb x\rb = x^{k-1} - x^k
    \end{equation*}
    gets its maximum at 
    \begin{equation}\label{23}
        f^{'}\lb x\rb
        = \lb k-1\rb x^{k-2} - k x^{k-1} 
        \implies x = 1 - \frac{1}{k}
    \end{equation}
    combining 
    \begin{align*}
        &
            UB_{k+1}-LB_{k+1}
        \\ & 
        \,\leq\,
        \int_0^1 \lb 1-\widehat{F}_k\rb \lc \lb \widehat{F}_k + \epsilon\rb^k - \lb \widehat{F}_k - \epsilon\rb^k \rc\; dw
        + 2\epsilon    
        \\ & \quad\quad\quad\quad
        \,=\, I + 2\epsilon 
        \,\leq\ \frac{2}{e} \epsilon + 2 \epsilon 
        \,=\, \lb \frac{2}{e} + 2\rb \epsilon
    \end{align*}
    as needed $\hfill\square$
 \end{proof}
 
  

To conclude we showed an empirical estimator for the 
expected gain from the $k+1$ hyperparameter sample 
and bound the error with $O\lb \sqrt{\frac{\log k}{k}}\rb$.

\subsection{Estimator development}
The next theorem presents a closed and more explicit 
formula for the empirical estimator defined thus far (\ref{thm_11}) 
\begin{theorem}
    Let $X_1, \hdots, X_n \sim \mathcal{D}$ be $n$ i.i.d random variables 
    we will denote the estimate gain from the $n+1$ sample 
    as defined in theorem (\ref{thm_11}) by 
    \begin{equation}\label{eq_int}
        S_n = \int_0^1\, F_n^n\cdot \lb 1 - F_n\rb \, d t 
    \end{equation}
    Assuming that $X_1, \hdots, X_n$ are given in their order 
    $X_1 \leq \hdots \leq X_n$ 
    and taking $X_{n+1}=1$. 
    Then a closed form representation for $S_n$ can be 
     \begin{equation}
            S_n 
        \,=\, 
            \overset{n}{\underset{k=1}{\Sigma}}\, 
            \lb \lc \frac{k}{n}\rc^n - \lc \frac{k}{n}\rc^{n+1} \rb 
            \cdot \lb X_{k+1} - X_k \rb 
     \end{equation}
\end{theorem}
\begin{proof}
    Assuming $X_1 \leq \hdots \leq X_n$ 
    and denoting $I_k = \lc X_k,\, X_{k+1}\rb$ 
    we have 
    \begin{equation*}
            F_n\lb x\rb 
        \,=\,
            \frac{1}{n}\,\overset{n}{\underset{k=1}{\Sigma}}\, 
            \mathbf{1}_{X_k \leq x} 
        \,=\,
            \frac{1}{n}\,\overset{n}{\underset{k=1}{\Sigma}}\, 
              k \cdot \mathbf{1}_{x \in I_k} 
    \end{equation*}
    and since all the $\ls I_k\rs_{k=1}^n$ are mutually disjoint we have 
    \begin{equation*}
            F_n\lb x\rb^n 
        \,=\,
            \lb \frac{1}{n}\,\overset{n}{\underset{k=1}{\Sigma}}\, 
              k \cdot \mathbf{1}_{x \in I_k} \rb^n 
        \,=\,
            \overset{n}{\underset{k=1}{\Sigma}}\, 
            \lc \frac{k}{n}\rc^n \cdot \mathbf{1}_{x \in I_k}
    \end{equation*}
    thus denoting by $I$ the integrand of equation (\ref{eq_int}) 
    i.e. 
    \begin{align*}
        & 
            I 
            \,=\, F_n^n\cdot \lb 1 - F_n\rb 
        \\ & 
            \,=\, 
            \lb \overset{n}{\underset{k=1}{\Sigma}}\, 
              \lc \frac{k}{n}\rc^n \cdot \mathbf{1}_{x \in I_k}\rb 
            \cdot 
            \lb 1 - \overset{n}{\underset{k=1}{\Sigma}}\, 
              \lc \frac{k}{n}\rc \cdot \mathbf{1}_{x \in I_k}\rb 
        \\ & 
            \,=\, 
                \lb \overset{n}{\underset{k=1}{\Sigma}}\, 
                \lc \frac{k}{n}\rc^n \cdot \mathbf{1}_{x \in I_k}\rb 
        \\ & 
            - 
            \lb \overset{n}{\underset{k=1}{\Sigma}}\, 
              \lc \frac{k}{n}\rc^n \cdot \mathbf{1}_{x \in I_k}\rb
              \cdot 
            \lb \overset{n}{\underset{k=1}{\Sigma}}\, 
              \lc \frac{k}{n}\rc \cdot \mathbf{1}_{x \in I_k}\rb 
        \\ & 
            \underbrace{\,=\,}_\textit{mutually disjoint} 
            \lb \overset{n}{\underset{k=1}{\Sigma}}\, 
              \lc \frac{k}{n}\rc^n \cdot \mathbf{1}_{x \in I_k}\rb 
             - 
            \lb \overset{n}{\underset{k=1}{\Sigma}}\, 
            \lc \frac{k}{n}\rc^{n+1} \cdot \mathbf{1}_{x \in I_k}\rb
        \\ & 
            \underbrace{\,=\,}_\textit{mutually disjoint} 
            \overset{n}{\underset{k=1}{\Sigma}}\, 
            \lb \lc \frac{k}{n}\rc^n - \lc \frac{k}{n}\rc^{n+1} \rb 
            \cdot \mathbf{1}_{x \in I_k}
    \end{align*}
    overall we get  
    \begin{align*}
        & 
            I 
            \,=\, 
            \overset{n}{\underset{k=1}{\Sigma}}\, 
            \lb \lc \frac{k}{n}\rc^n - \lc \frac{k}{n}\rc^{n+1} \rb 
            \cdot \mathbf{1}_{x \in I_k}
    \end{align*}
    Taking the integral we end with 
    \begin{align*}
        & 
        S_n 
        \,=\, \int_0^1 F_n^n\cdot \lb 1 - F_n\rb \, d x  
        \\ & 
        \,=\, 
            \int_0^1 \overset{n}{\underset{k=1}{\Sigma}}\, 
            \lb \lc \frac{k}{n}\rc^n - \lc \frac{k}{n}\rc^{n+1} \rb 
            \cdot \mathbf{1}_{t \in I_k} \; d x
        \\ & 
        \,=\, 
            \overset{n}{\underset{k=1}{\Sigma}}\, 
            \lb \lc \frac{k}{n}\rc^n - \lc \frac{k}{n}\rc^{n+1} \rb 
            \cdot
            \int_0^1  \mathbf{1}_{x \in I_k} \; d x
        \\ & 
        \,=\, 
            \overset{n}{\underset{k=1}{\Sigma}}\, 
            \lb \lc \frac{k}{n}\rc^n - \lc \frac{k}{n}\rc^{n+1} \rb 
            \cdot \abs{I_k}
        \\ & 
        \,=\, 
            \overset{n}{\underset{k=1}{\Sigma}}\, 
            \lb \lc \frac{k}{n}\rc^n - \lc \frac{k}{n}\rc^{n+1} \rb 
            \cdot \lb X_{k+1} - X_k \rb 
    \end{align*}
    To sum up we ended with 
    \begin{align*}
        & 
            S_n 
        \,=\, 
            \overset{n}{\underset{k=1}{\Sigma}}\, 
            \lb \lc \frac{k}{n}\rc^n - \lc \frac{k}{n}\rc^{n+1} \rb 
            \cdot \lb X_{k+1} - X_k \rb 
    \end{align*}
    as needed $\hfill\square$
\end{proof}

\section{Lower bound}
\paragraph{}
In the previous section we established an estimator with bounded error on the expected gain from the $k+1$ hyperparameter tuning step. 
In this section we will bound from below the best error possible for such estimators. 
We will do so by bounding from below the error for
$X_1, \hdots,X_k \sim \mathcal{N}\lb \mu,\, \sigma\rb$ 
and thus derive a lower bound on the best possible error for 
$X_1, \hdots, X_k$ taken from a general distribution 

\paragraph{}
Let $X$ be normally distributed r.v. $ X \sim \mathcal{N}\lb \mu,\, \sigma\rb$ 
and let $\alpha$ be our highest achieved accuracy 
then the expected improvement from additional sample  
 is $\mathbb{E}\lc X \,|\, X > \alpha\,\rc\,$ 
the next lemma provides estimate for this quantity 

\begin{lemma}\label{lem_3}
    For $X \sim \mathcal{N}\lb \mu, \sigma\rb$ 
    and $f_X,\, F_X$ the density and accumulative distribution 
    functions for $X$ the following equality holds 
    \begin{equation*}
        \mathbb{E}\lc X \,|\, X > \alpha\,\rc 
        \,=\, 
        \mu + 
        \sigma^2 
        \frac{f_X\lb \alpha\rb}{1-F_X\lb \alpha\rb}
    \end{equation*}
\end{lemma}
\begin{proof}
    Well known as the inverse Mills ratio 
    \cite{greene2003econometric} 
    $\hfill\square$
\end{proof}

\begin{lemma}\label{lem_03}
    For $X\sim \mathcal{N}\lb \mu, \sigma\rb$ 
    a gaussian distribution and $\alpha \in \mathbb{R}$ 
    the current accuracy. 
    The expected improvement from 
    sampling one more point 
    is 
    $\mathbb{E}\lc\, 
    \lb X-\alpha\rb \cdot \mathbf{1}_{X \geq \alpha}
    \,\rc$
    and the following equality holds 
    \begin{equation}\label{eq_33}
            \mathbb{E}\lc\, 
                \lb X-\alpha\rb \cdot \mathbf{1}_{X \geq \alpha}
            \,\rc 
        \,=\, 
        \lb \mu - \alpha\rb
        \cdot
        \lb 1- F_X\lb \alpha \rb \rb
        \,+\,
        \sigma^2 \cdot 
        f_X\lb \alpha\rb 
    \end{equation}
\end{lemma}

\begin{proof}
    Developing the right hand side of (\ref{eq_33}) we have
    \begin{equation*}
            \mathbb{E}\lc\, 
                \lb X-\alpha\rb \cdot \mathbf{1}_{X \geq \alpha}
            \,\rc 
        \,=\, 
            \mathbb{E}\lc\, 
                \lb X-\alpha\rb \,|\,X \geq \alpha 
            \,\rc \cdot P\lb X \geq \alpha \rb 
    \end{equation*}
    \begin{equation*}
        \,=\,
            \lb \mathbb{E}\lc\, X \,|\,X 
            \geq \alpha\,\rc - \alpha\rb
            \cdot
            \lb 1- F_X\lb \alpha \rb \rb  
    \end{equation*}
    \begin{equation*}
        \underbrace{\,=\,}_\textit{Lemma (\ref{lem_3})}
            \lb 
                \mu - \alpha +
                \sigma^2 
                \frac{f_X\lb \alpha\rb}{1-F_X\lb \alpha\rb}  
            \rb
            \cdot
            \lb 1- F_X\lb \alpha \rb \rb  
    \end{equation*}
    \begin{equation*}
        \,=\,
            \lb \mu - \alpha\rb
            \cdot
            \lb 1- F_X\lb \alpha \rb \rb
            +
            \sigma^2 \cdot 
            f_X\lb \alpha\rb 
    \end{equation*}
    as needed $\hfill\hfill\square$
\end{proof}

\paragraph{}
Denoting by $\Hat{\mu},\,\Hat{\sigma}$ the empirical 
values for the gaussian $\mathcal{N}\lb \mu,\, \sigma\rb$ 
and noting that $\Hat{\mu},\,\Hat{\sigma}$ are sufficient 
statistics for the gaussian distribution the best 
estimator for the expected improvement 
from additional sampler is 
\begin{equation*}
        err = I\lb \Hat{\mu}, \Hat{\sigma}, \alpha\rb
\end{equation*}
when 
\begin{equation*}
        I\lb \mu, \sigma, \alpha\rb 
    = 
        \lb \mu - \alpha\rb
        \cdot
        \lb 1- F_X\lb \alpha \rb \rb
        +
        \sigma^2 \cdot 
        f_X\lb \alpha\rb 
\end{equation*}
and $\alpha = \max\lb X_1, \hdots, X_n\rb\,$.  
Thus the discrepancy is lower bounded by 
\begin{align*}
    & 
            err
        \,=\,
            I\lb \mu, \sigma, \alpha\rb - I\lb \Hat{\mu}, \Hat{\sigma}, \alpha\rb
\end{align*}
and the total discrepancy in terms of the mean absolute error is  
\begin{align*}
    & 
            err_{ma} 
        \,=\, 
            \underset{\Hat{\mu},\, \Hat{\sigma}}{\mathbb{E}}\lc\, 
                \abs{err}
            \,\rc 
        \,=\,
            \underset{\Hat{\mu},\, \Hat{\sigma}}{\mathbb{E}}\lc\,
                \abs{\,
                    I\lb \mu, \sigma, \alpha\rb - 
                    I\lb \Hat{\mu}, \Hat{\sigma}, \alpha\rb
                \,}
            \,\rc
\end{align*}

Next we will lower bound the mean absolute error 
$err_{ma}\,$ and for that we will start 
with the following lemma 

\begin{lemma}\label{lem_lb_1} 
    $I$ is monotonic increasing in both 
    $\mu,\, \sigma$, more formally 
    \begin{align*}
        & 
        \frac{\partial}{\partial \mu}\, 
        I\lb \mu, \sigma, \alpha\rb \,\geq\, 0 
        \quad\textit{and}\quad 
        \frac{\partial}{\partial \sigma}\, 
        I\lb \mu, \sigma, \alpha\rb \,\geq\, 0 
    \end{align*}
\end{lemma}
\begin{proof}
    Deferred to appendix A  $\hfill\square$
\end{proof}

\begin{proposition}\label{prop_1}
    In particular we have
    \begin{align*}
    & \quad\quad
        \frac{\partial}{\partial \mu}\, 
        I\lb \mu, \sigma, \alpha\rb 
        \,=\, 1- F_{Z}\lb \frac{\alpha -\mu}{\sigma}\rb 
    \\ & 
        \textit{and}
    \\ & \quad\quad
        \frac{\partial}{\partial \sigma}\, 
        I\lb \mu, \sigma, \alpha\rb 
       \,=\, 
       \lc\, 2 \cdot \lb \frac{\alpha -\mu}{\sigma}\rb^2 + 1 \,\rc \cdot f_{Z}\lb \frac{\alpha -\mu}{\sigma}\rb 
    \end{align*}
\end{proposition}   
\begin{proof}
    Follows immediately from the proof of the last lemma (\ref{lem_lb_1}) 
    $\hfill\square$
\end{proof}

\begin{lemma}\label{lem_conv_1}
    $I$ is always convex in $\alpha$ 
    it is also convex in $\sigma$ whenever 
    \begin{align*}
         \alpha   
        \;\leq\; \mu - \sqrt{\frac{3}{2}}\sigma 
        \;\;\;\textit{or}\;\;\; 
        \alpha   
        \;\geq\; \mu + \sqrt{\frac{3}{2}}\sigma\,   
    \end{align*}
    more formally 
    \begin{align*}
        & 
        \frac{\partial^2}{\partial^2 \mu}\, 
        I\lb \mu, \sigma, \alpha\rb \,\geq\, 0 
        \quad\textit{and}\quad
        \frac{\partial^2}{\partial^2 \sigma}\, 
        I\lb \mu, \sigma, \alpha\rb \,\geq\, 0 
    \end{align*}
    when the second inequality holds in the specified range. 
\end{lemma}
\begin{proof}
    Deferred to appendix B $\hfill\square$
\end{proof}

\begin{lemma}\label{lem_4}
    The following equality holds for the discrepancy 
    $\Hat{\mu} - \mu$
    \begin{align*}
        \underset{\Hat{\mu}}{\mathbb{E}}\lc\,
              \Hat{\mu} - \mu  
            \;|\; \Hat{\mu} - \mu > 0 
        \,\rc  
    \,=\, 
        \sqrt{\frac{2}{ \pi}}
        \cdot\frac{\sigma^2}{n}   
    \end{align*}
\end{lemma}
\begin{proof}
    Observe that 
    $
    \Hat{\mu} = \overline{X_n} 
    \,\sim\, 
    \mathcal{N}\lb \mu,\, \frac{\sigma}{\sqrt{n}}\rb
    \,$ 
    and thus 
    $
    \Hat{\mu} - \mu  
    \sim \mathcal{N}\lb 0,\, \frac{\sigma}{\sqrt{n}}\rb
    $ 
    denoting $Y = \Hat{\mu} - \mu$ and 
    using lemma (\ref{lem_3}) we have 
    \begin{align*}
        & 
            \mathbb{E}\lc Y \,|\, Y > \alpha\,\rc 
            \,=\, 
            \mu_Y + 
            \sigma_Y^2 
            \frac{f_Y\lb \alpha\rb}{1-F_Y\lb \alpha\rb}
    \end{align*}
    taking $Y=\Hat{\mu} - \mu$ and $\alpha = 0$ results in 
    \begin{align*}
        &         
            \underset{\Hat{\mu}}{\mathbb{E}}\lc\,
                \Hat{\mu} - \mu  
                \;|\; \Hat{\mu} - \mu > 0 
            \,\rc  
        \,=\, 
            \mathbb{E}\lc Y \,|\, Y > \alpha\,\rc 
        \\ & 
        \,=\, 
            \mu_Y + 
            \sigma_Y^2 
            \frac{f_Y\lb \alpha\rb}{1-F_Y\lb \alpha\rb}
        \,=\, 
            0 + 
            \lb \frac{\sigma}{\sqrt{n}}\rb^2 
            \frac{f_Z\lb 0\rb}{1-F_Z\lb 0\rb}
        \\ & 
        \,=\, 
            \frac{\sigma^2}{n} \cdot 
            \frac{\frac{1}{\sqrt{2 \pi}}}{1-\frac{1}{2}}
        \,=\, 
            \frac{\sigma^2}{n} \cdot 
            \frac{2}{\sqrt{2 \pi}}
        \,=\, 
            \sqrt{\frac{2}{ \pi}}
            \cdot\frac{\sigma^2}{n} 
            \quad\quad\quad\;\;\;\square
    \end{align*}
\end{proof}

The next lemma establish lower bound on the mean absolute error
$err_{ma}$ 
of the best estimator for the expected improvement 
by one more hyperparameter tuning step  

\begin{lemma} 
    The following lower bound on $err_{ma}$ holds   
    \begin{align*}
        & 
                err_{ma} 
            \,=\, 
                \underset{\Hat{\mu},\, \Hat{\sigma}}{\mathbb{E}}\lc\, 
                    \abs{err}
                \,\rc 
            \,=\,
                \underset{\Hat{\mu},\, \Hat{\sigma}}{\mathbb{E}}\lc\,
                    \abs{\,I\lb \mu, \sigma, \alpha\rb - I\lb \Hat{\mu}, \Hat{\sigma}, \alpha\rb\,}
                \,\rc 
        \\ & \quad\quad\quad\quad\quad\quad\quad\quad\;
            \,\geq\,
                \frac{\sigma^2}{2 \,\sqrt{2\, \pi}\, n} 
                \cdot 
               \lc\, 1 - F_X\lb \alpha\rb \,\rc 
    \end{align*}
\end{lemma}

\begin{proof}
    \begin{align*}
        &
            err_{ma} 
        \,=\,
            \underset{\Hat{\mu},\, \Hat{\sigma}}{\mathbb{E}}\lc\,
                \abs{\,I\lb \mu, \sigma, \alpha\rb - I\lb \Hat{\mu}, \Hat{\sigma}, \alpha\rb\,}
            \,\rc 
        \\ & 
        \,\geq\, 
            \underset{\Hat{\mu},\, \Hat{\sigma}}{\mathbb{E}}\lc\,
                \abs{\,I\lb \mu, \sigma, \alpha\rb - I\lb \Hat{\mu}, \Hat{\sigma}, \alpha\rb\,} 
                \;|\; \Hat{\mu} > \mu \quad\textit{and}\quad \Hat{\sigma} > \sigma 
            \,\rc  
        \\ & \quad\;\;
            \cdot
            P\lb \Hat{\mu} > \mu \quad\textit{and}\quad \Hat{\sigma} > \sigma\rb
        \\ &  
            \underbrace{\,=\,}_\textit{symmetry}
            \frac{1}{4} \cdot \underset{\Hat{\mu},\, \Hat{\sigma}}{\mathbb{E}}\lc\,
                \abs{\,I\lb \mu, \sigma, \alpha\rb - I\lb \Hat{\mu}, \Hat{\sigma}, \alpha\rb\,} 
                \;|\; \Hat{\mu} > \mu \,,\; \Hat{\sigma} > \sigma 
            \,\rc 
        \\&  
            \underbrace{\,=\,}_\textit{lemma (\ref{lem_lb_1})}
            \frac{1}{4} \cdot \underset{\Hat{\mu},\, \Hat{\sigma}}{\mathbb{E}}\lc\,
                \,I\lb \Hat{\mu}, \Hat{\sigma}, \alpha\rb - I\lb \mu, \sigma, \alpha\rb \,
                \;|\; \Hat{\mu} > \mu\,,\; \Hat{\sigma} > \sigma 
            \,\rc 
        \\ &  
            \,=\,
            \frac{1}{4} \cdot\lc\, \underset{\Hat{\mu},\, \Hat{\sigma}}{\mathbb{E}}\lc\,
                \,I\lb \Hat{\mu}, \Hat{\sigma}, \alpha\rb\,
                \;|\; \Hat{\mu} > \mu \,,\; \Hat{\sigma} > \sigma 
            \,\rc - I\lb \mu, \sigma, \alpha\rb \,\rc   
        \\ &  
            \underbrace{\,\geq\,}_\textit{lemma (\ref{lem_lb_1})} 
            \frac{1}{4} \cdot\lc\, \underset{\Hat{\mu},\, \Hat{\sigma}}{\mathbb{E}}\lc\,
                \,I\lb \Hat{\mu}, \sigma, \alpha\rb\,
                \;|\; \Hat{\mu} > \mu \,,\; \Hat{\sigma} > \sigma 
            \,\rc - I\lb \mu, \sigma, \alpha\rb \,\rc 
        \\ &  
            \,=\, 
            \frac{1}{4} \cdot\lc\, \underset{\Hat{\mu}}{\mathbb{E}}\lc\,
                \,I\lb \Hat{\mu}, \sigma, \alpha\rb\,
                \;|\; \Hat{\mu} > \mu 
            \,\rc - I\lb \mu, \sigma, \alpha\rb \,\rc 
        \\ &  
            \,=\, 
            \frac{1}{4} \cdot \underset{\Hat{\mu}}{\mathbb{E}}\lc\,
                I\lb \Hat{\mu}, \sigma, \alpha\rb
                 - I\lb \mu, \sigma, \alpha\rb
                \;|\; \Hat{\mu} > \mu 
            \,\rc  
        \\ &  
            \underbrace{\,\geq\,}_\textit{lemma (\ref{lem_conv_1})} 
            \frac{1}{4} \cdot \underset{\Hat{\mu}}{\mathbb{E}}\lc\,
                 \frac{\partial}{\partial \mu} I\lb \mu, \sigma, \alpha\rb 
                 \cdot \lb \Hat{\mu} - \mu\rb 
                \;|\; \Hat{\mu} > \mu 
            \,\rc  
        \\ &  
            \,=\, 
            \frac{1}{4} \cdot 
            \frac{\partial}{\partial \mu}\, I\lb \mu, \sigma, \alpha\rb
            \cdot 
            \underset{\Hat{\mu}}{\mathbb{E}}\lc\,
                  \Hat{\mu} - \mu  
                \;|\; \Hat{\mu} > \mu 
            \,\rc  
        \\&  
            \,=\, 
            \frac{1}{4} \cdot 
            \frac{\partial}{\partial \mu}\, I\lb \mu, \sigma, \alpha\rb
            \cdot 
            \underset{\Hat{\mu}}{\mathbb{E}}\lc\,
                  \Hat{\mu} - \mu  
                \;|\; \Hat{\mu} - \mu > 0 
            \,\rc  
        \\ & 
            \underbrace{\,=\,}_\textit{lemma (\ref{lem_4})}   
            \frac{1}{4} \cdot 
            \frac{\partial}{\partial \mu}\, I\lb \mu, \sigma, \alpha\rb
            \cdot 
            \sqrt{\frac{2}{\pi}}\cdot\frac{\sigma^2}{n}
        \\ & 
            \underbrace{\,=\,}_\textit{proposition (\ref{prop_1})}   
            \frac{\sigma^2}{2 \,\sqrt{2\, \pi}\, n} 
            \cdot 
            \lc\, 1- F_{Z}\lb \frac{\alpha -\mu}{\sigma}\rb \,\rc
        \\ & 
            \,=\,   
            \frac{\sigma^2}{2 \,\sqrt{2\, \pi}\, n} 
            \cdot 
           \lc\, 1 - F_X\lb \alpha\rb \,\rc 
           \quad\quad\quad\quad\quad\;\;\;
           \quad\quad\quad\quad\quad\;\;\;
           \hfill\square
    \end{align*}
\end{proof}

\begin{lemma} 
    In the region
    $\alpha \;\geq\; \mu + \sqrt{\frac{3}{2}}\cdot\sigma\,$
    the following lower bound holds 
    \begin{align*}
        & 
                err_{ma} 
            \,=\, 
                \underset{\Hat{\mu},\, \Hat{\sigma}}{\mathbb{E}}\lc\, 
                    \abs{err}
                \,\rc 
            \,=\,
                \underset{\Hat{\mu},\, \Hat{\sigma}}{\mathbb{E}}\lc\,
                    \abs{\,I\lb \mu, \sigma, \alpha\rb - I\lb \Hat{\mu}, \Hat{\sigma}, \alpha\rb\,}
                \,\rc 
        \\ & \quad\quad\quad\quad\quad\quad\;
            \,\geq\,
           \frac{\sigma}{2 \cdot e^\frac{3}{2}\cdot n} 
           \cdot \lb \frac{\alpha -\mu}{\sigma}\rb^2 
            \cdot f_{X}\lb \alpha\rb 
    \end{align*}
\end{lemma}

\begin{proof}
    \begin{align*}
        &
            err_{ma} 
        \,=\,
            \underset{\Hat{\mu},\, \Hat{\sigma}}{\mathbb{E}}\lc\,
                \abs{\,I\lb \mu, \sigma, \alpha\rb - I\lb \Hat{\mu}, \Hat{\sigma}, \alpha\rb\,}
            \,\rc 
        \\ & 
        \,\geq\, 
            \underset{\Hat{\mu},\, \Hat{\sigma}}{\mathbb{E}}\lc\,
                \abs{\,I\lb \mu, \sigma, \alpha\rb - I\lb \Hat{\mu}, \Hat{\sigma}, \alpha\rb\,} 
                \;|\; \Hat{\mu} > \mu \quad\textit{and}\quad \Hat{\sigma} > \sigma 
            \,\rc  
        \\ & \quad\;\;
            \cdot
            P\lb \Hat{\mu} > \mu \quad\textit{and}\quad \Hat{\sigma} > \sigma\rb
        \\ &  
            \underbrace{\,=\,}_\textit{symmetry}
            \frac{1}{4} \cdot \underset{\Hat{\mu},\, \Hat{\sigma}}{\mathbb{E}}\lc\,
                \abs{\,I\lb \mu, \sigma, \alpha\rb - I\lb \Hat{\mu}, \Hat{\sigma}, \alpha\rb\,} 
                \;|\; \Hat{\mu} > \mu \,,\; \Hat{\sigma} > \sigma 
            \,\rc 
        \\&  
            \underbrace{\,=\,}_\textit{lemma (\ref{lem_lb_1})}
            \frac{1}{4} \cdot \underset{\Hat{\mu},\, \Hat{\sigma}}{\mathbb{E}}\lc\,
                \,I\lb \Hat{\mu}, \Hat{\sigma}, \alpha\rb - I\lb \mu, \sigma, \alpha\rb \,
                \;|\; \Hat{\mu} > \mu\,,\; \Hat{\sigma} > \sigma 
            \,\rc 
        \\&  
            \,=\,
            \frac{1}{4} \cdot\lc\, \underset{\Hat{\mu},\, \Hat{\sigma}}{\mathbb{E}}\lc\,
                \,I\lb \Hat{\mu}, \Hat{\sigma}, \alpha\rb\,
                \;|\; \Hat{\mu} > \mu \,,\; \Hat{\sigma} > \sigma 
            \,\rc - I\lb \mu, \sigma, \alpha\rb \,\rc   
        \\ &  
            \underbrace{\,\geq\,}_\textit{lemma (\ref{lem_lb_1})} 
            \frac{1}{4} \cdot\lc\, \underset{\Hat{\mu},\, \Hat{\sigma}}{\mathbb{E}}\lc\,
                \,I\lb \mu, \Hat{\sigma}, \alpha\rb\,
                \;|\; \Hat{\mu} > \mu \,,\; \Hat{\sigma} > \sigma 
            \,\rc - I\lb \mu, \sigma, \alpha\rb \,\rc 
        \\ &  
            \,=\, 
            \frac{1}{4} \cdot\lc\, \underset{\Hat{\sigma}}{\mathbb{E}}\lc\,
                \,I\lb \mu, \Hat{\sigma}, \alpha\rb\,
                \;|\; \Hat{\sigma} > \sigma 
            \,\rc - I\lb \mu, \sigma, \alpha\rb \,\rc 
        \\ &  
            \,=\, 
            \frac{1}{4} \cdot \underset{\Hat{\sigma}}{\mathbb{E}}\lc\,
                I\lb \mu, \Hat{\sigma}, \alpha\rb
                 - I\lb \mu, \sigma, \alpha\rb
                \;|\; \Hat{\sigma} > \sigma 
            \,\rc  
        \\ &  
            \underbrace{\,\geq\,}_\textit{lemma (\ref{lem_conv_1})} 
            \frac{1}{4} \cdot \underset{\Hat{\sigma}}{\mathbb{E}}\lc\,
                 \frac{\partial}{\partial \sigma} 
                 I\lb \mu, \sigma, \alpha\rb 
                 \cdot \lb \Hat{\sigma} - \sigma\rb 
                \;|\; \Hat{\sigma} > \sigma 
            \,\rc  
        \\ &  
            \,=\, 
            \frac{1}{4} \cdot 
            \frac{\partial}{\partial \sigma}\, I\lb \mu, \sigma, \alpha\rb
            \cdot 
            \underset{\Hat{\sigma}}{\mathbb{E}}\lc\,
                  \Hat{\sigma} - \sigma  
                \;|\; \Hat{\sigma} > \sigma
            \,\rc  
        \\&  
            \,=\, 
            \frac{1}{4} \cdot 
            \frac{\partial}{\partial \sigma}\, I\lb \mu, \sigma, \alpha\rb
            \cdot 
            \underset{\Hat{\sigma}}{\mathbb{E}}\lc\,
                  \Hat{\sigma} - \sigma  
                \;|\; \Hat{\sigma} - \sigma > 0
            \,\rc  
        \\&  
            \underbrace{\,\geq\,}_\textit{lemma (\ref{lem_5})}
            \frac{\sigma}{4 \cdot e^\frac{3}{2}\cdot n} \cdot 
            \frac{\partial}{\partial \sigma}\, I\lb \mu, \sigma, \alpha\rb
        \\ & 
            \underbrace{\,=\,}_\textit{(\ref{prop_1})}
           \frac{\sigma}{4 \cdot e^\frac{3}{2}\cdot n} 
           \cdot
           \lc\, 2 \cdot \lb \frac{\alpha -\mu}{\sigma}\rb^2 + 1 \,\rc 
            \cdot f_{Z}\lb \frac{\alpha -\mu}{\sigma}\rb
        \\ & 
            \,\geq\,
           \frac{\sigma}{2 \cdot e^\frac{3}{2}\cdot n} 
           \cdot \lb \frac{\alpha -\mu}{\sigma}\rb^2 
            \cdot f_{Z}\lb \frac{\alpha -\mu}{\sigma}\rb 
        \\ & 
           \,=\,
           \frac{\sigma}{2 \cdot e^\frac{3}{2}\cdot n} 
           \cdot \lb \frac{\alpha -\mu}{\sigma}\rb^2 
            \cdot f_{X}\lb \alpha\rb 
    \end{align*}
    as needed $\hfill\square$ 
\end{proof}

\begin{lemma}\label{lem_5}
    The following lower bound holds for the discrepancy 
    $\Hat{\sigma} - \sigma$
    \begin{align*}
        \underset{\Hat{\sigma}}{\mathbb{E}}\lc\,
              \Hat{\sigma} - \sigma  
            \;|\; \Hat{\sigma} - \sigma > 0 
        \,\rc  
    \,\geq\, 
       \frac{\sigma}{e^{\frac{3}{2}} \cdot n}
    \end{align*}
\end{lemma}
\begin{proof}
    By $\Hat{\sigma}$ definition we have    
    \begin{align*}
        &
        \Hat{\sigma}^2
        = 
        \frac{1}{n-1} 
        \overset{n}{\underset{i=1}{\Sigma}}\, 
        \lb x_i - \overline{x}\rb^2 
        \,\sim\, 
        \frac{\sigma^2}{n-1}\; \chi_{n-1}^2  
    \end{align*}
    when $\chi_{n-1}^2$ is the distribution defined by 
    \begin{align*}
        &
            f_{\chi_{n-1}^2}\lb x\rb 
        \,=\,
            \frac{x^{\frac{k}{2}-1} e^{-\frac{x}{2}}}{\Gamma\lb \frac{k}{2}\rb 2^{\frac{k}{2}}} 
    \end{align*} 
    and thus 
    \begin{align*}
        &
        \Hat{\sigma}
        = 
        \sqrt{\frac{1}{n-1} 
        \overset{n}{\underset{i=1}{\Sigma}}\, 
        \lb x_i - \overline{x}\rb^2}
        \,\sim\, 
        \sigma \cdot \sqrt{\frac{\,\chi_{n-1}^2}{n-1}} 
    \end{align*}
    substituting we have 
    \begin{align*}
    &
        \underset{\Hat{\sigma}}{\mathbb{E}}\lc\,
              \Hat{\sigma} - \sigma  
            \;|\; \Hat{\sigma} - \sigma > 0 
        \,\rc  
    \\ & 
    \,=\, 
        \underset{\Hat{\sigma}}{\mathbb{E}}\lc\,
            \sigma \cdot \sqrt{\frac{\,\chi_{n-1}^2}{n-1}} \,-\, \sigma  
            \;|\;\, \sqrt{\frac{\chi_{n-1}^2}{n-1}} \,>\, 1 
        \,\rc  
    \\ & 
    \,=\, 
        \sigma \cdot 
        \underset{\chi_{n-1}^2}{\mathbb{E}}\lc\,
            \sqrt{\frac{\chi_{n-1}^2}{n-1}} - 1  
            \;|\;\, \sqrt{\frac{\chi_{n-1}^2}{n-1}} - 1\,>\, 0 
        \,\rc  
    \\ & 
    \,=\, 
        \frac{\sigma}{1-F_{\sqrt{\frac{\chi_{n-1}^2}{n-1}}}\lb 1\rb} 
        \cdot 
        \underset{\chi_{n-1}^2}{\mathbb{E}}\lc\,
            \lb\, \sqrt{\frac{\chi_{n-1}^2}{n-1}} - 1 \,\rb^{+} 
        \,\rc  
    \\ & 
    \underbrace{\,\geq\,}_{\lb 4 \rb} 
        \frac{\sigma}{1-F_{\sqrt{\frac{\chi_{n-1}^2}{n-1}}}\lb 1\rb} 
        \cdot \gamma \cdot \lc\, 1 - F_{\sqrt{\frac{\chi_{n-1}^2}{n-1}}}\lb 1+\gamma\rb \,\rc 
    \\ & 
    \underbrace{\,=\,}_\textit{lemma (\ref{lem_21})}
        \sigma\cdot \gamma \cdot
         \frac{1 - F_{\chi_{n-1}^2}\lb \lb\, n-1 \,\rb \lb\, 1+\gamma \,\rb^2\rb }
        {1-F_{\chi_{n-1}^2}\lb\, n-1 \,\rb}
    \\ & 
    \,=\,
        \sigma\cdot \gamma \cdot
         \frac{1 - \frac{\gamma\lb \frac{n-1}{2}, \frac{\lb\, n-1 \,\rb \lb\, 1+\gamma \,\rb^2}{2}\rb}
         {\Gamma\lb \frac{n-1}{2}\rb}}
        {1 - \frac{\gamma\lb \frac{n-1}{2}, \frac{ n-1 }{2}\rb}
         {\Gamma\lb \frac{n-1}{2}\rb}}
    \\ & 
    \,=\,
        \sigma\cdot \gamma \cdot
         \frac{\Gamma\lb \frac{n-1}{2}\rb - 
         \gamma\lb \frac{n-1}{2}, \frac{\lb\, n-1 \,\rb \lb\, 1+\gamma \,\rb^2}{2}\rb}
        {\Gamma\lb \frac{n-1}{2}\rb - \gamma\lb \frac{n-1}{2}, \frac{ n-1 }{2}\rb}
    \\ & 
    \,=\,
        \sigma\cdot \gamma \cdot
         \frac{\Gamma\lb \frac{n-1}{2}, \frac{\lb\, n-1 \,\rb \cdot \lb\, 1+\gamma \,\rb^2}{2}\rb}
        {\Gamma\lb \frac{n-1}{2}, \frac{ n-1 }{2}\rb}
    \\ & 
    \underbrace{\,\geq\,}_\textit{lemma (\ref{lem_23})} 
        \sigma\cdot \gamma \cdot 
        e^{- \lb \gamma + \frac{\gamma^2}{2}\rb \cdot\lb n-1\rb}
    \end{align*}
    when $\gamma \in \mathbb{R}_{>0}$ is arbitrary constant and 
    transition $\lb 4\rb$ is due to the markov inequality,        
    writing again for clarity  
    \begin{align*}
        \underset{\Hat{\sigma}}{\mathbb{E}}\lc\,
              \Hat{\sigma} - \sigma  
            \;|\; \Hat{\sigma} - \sigma > 0 
        \,\rc  
        \,\geq\, 
        \sigma\cdot \gamma \cdot 
        e^{- \lb \gamma + \frac{\gamma^2}{2}\rb \cdot\lb n-1\rb}
    \end{align*}
    choosing $\gamma=\frac{1}{n}$ we end with 
    \begin{align*}
        \underset{\Hat{\sigma}}{\mathbb{E}}\lc\,
              \Hat{\sigma} - \sigma  
            \;|\; \Hat{\sigma} - \sigma > 0 
        \,\rc  
        \,\geq\, 
        \frac{\sigma}{e^{\frac{3}{2}} \cdot n} 
    \end{align*}
    as needed 
    $\hfill\square$
\end{proof}

The next lemma present $\sqrt{\frac{\chi_{n-1}^2}{n-1}}$ 
density and cumulative functions in terms of 
the corresponding $\chi_{n-1}^2$-density and cumulative functions

\begin{lemma}\label{lem_21}
     $\sqrt{\frac{\chi_{n-1}^2}{n-1}}$ cumulative function 
     can be computed by 
    \begin{align*}
        &
        F_{\sqrt{\frac{\chi_{n-1}^2}{n-1}}}\lb x\rb 
        \,=\, 
        F_{\chi_{n-1}^2}\lb\, \lc n-1\rc\cdot x^2 \,\rb
    \end{align*}
    and the density function by 
    \begin{align*}
        &
        f_{\sqrt{\frac{\chi_{n-1}^2}{n-1}}}\lb x\rb 
        \,=\, 
        2 \cdot \lb n-1\rb\cdot x \cdot f_{\chi_{n-1}^2}\lb\, \lc n-1\rc\cdot x^2 \,\rb
    \end{align*}
\end{lemma}
\begin{proof}
    \begin{align*}
        &
        F_{\sqrt{\frac{\chi_{n-1}^2}{n-1}}}\lb x\rb 
        \,=\, 
        \mathbb{P}\lc\; \sqrt{\frac{\chi_{n-1}^2}{n-1}} \,\leq\, x \;\rc 
        \,=\, 
        \mathbb{P}\lc\; \frac{\chi_{n-1}^2}{n-1} \,\leq\, x^2 \;\rc 
        \\ & 
        \,=\, 
        \mathbb{P}\lc\; \chi_{n-1}^2 \,\leq\, \lb n-1\rb\cdot x^2 \;\rc  
        \,=\, 
        F_{\chi_{n-1}^2}\lb\, \lc n-1\rc\cdot x^2 \,\rb 
    \end{align*}
    similarly 
    \begin{align*}
    &
        f_{\sqrt{\frac{\chi_{n-1}^2}{n-1}}}\lb x\rb 
        \,=\, 
        \frac{\partial}{\partial x}\, F_{\sqrt{\frac{\chi_{n-1}^2}{n-1}}}\lb x\rb
        \,=\, 
        \frac{\partial}{\partial x}\, F_{\chi_{n-1}^2}\lb\, \lc n-1\rc\cdot x^2 \,\rb
    \\ & 
        \,=\, 
        2 \lc n-1\rc\cdot x \cdot f_{\chi_{n-1}^2}\lb\, \lc n-1\rc\cdot x^2 \,\rb 
        \quad\quad\quad\quad\quad\quad\quad\quad\square
    \end{align*}
\end{proof}

\begin{lemma}\label{lem_23}
    The following inequality holds 
    \begin{align*}
        \forall s,\, x,\, y \in \mathbb{R}_{>0}\quad\quad
        \frac{\Gamma\lb s, x+y\rb}{\Gamma\lb s, x\rb} 
        \,\geq\, e^{-y} 
    \end{align*}
\end{lemma}
\begin{proof}
    Deferred to appendix C. $\hfill\square$
\end{proof}


\begin{conclusion}
    The mean absolute error $err_{ma}$ is bounded from below by 
    \begin{align*}
        & 
            err_{ma} 
        \,=\, 
            \underset{\Hat{\mu},\, \Hat{\sigma}}{\mathbb{E}}\lc\, 
                \abs{err}
            \,\rc 
        \,=\,
            \underset{\Hat{\mu},\, \Hat{\sigma}}{\mathbb{E}}\lc\,
                I\lb \mu, \sigma, \alpha\rb - I\lb \Hat{\mu}, \Hat{\sigma}, \alpha\rb
            \,\rc 
        \\ & \quad\quad\quad\quad\quad\quad\quad\quad\quad\;\;\;
        \,\geq\,            
            \frac{\sigma^2}{2 \,\sqrt{2\, \pi}\, n} 
            \cdot 
           \lc\, 1 - F_X\lb \alpha\rb \,\rc 
    \end{align*}
    and for $\alpha \,\geq\, \mu + \sqrt{\frac{3}{2}}\cdot \sigma\,$ 
    the bound can be improved to  
    \begin{align*}
        & 
            err_{ma} 
        \,=\, 
            \underset{\Hat{\mu},\, \Hat{\sigma}}{\mathbb{E}}\lc\, 
                \abs{err}
            \,\rc 
        \,=\,
            \underset{\Hat{\mu},\, \Hat{\sigma}}{\mathbb{E}}\lc\,
                I\lb \mu, \sigma, \alpha\rb - I\lb \Hat{\mu}, \Hat{\sigma}, \alpha\rb
            \,\rc 
        \\ & \quad\quad\quad\quad\quad\quad\quad\;\;
        \,\geq\,            
            \frac{\sigma}{2 \cdot e^\frac{3}{2}\cdot n} 
           \cdot \lb \frac{\alpha -\mu}{\sigma}\rb^2 
            \cdot f_{X}\lb \alpha\rb 
    \end{align*}
    fixing the problem we have the following 
    asymptotic bound in $n,\, \alpha$ 
    \begin{align*}
        & 
            err_{ma} 
        =
            \underset{\Hat{\mu},\, \Hat{\sigma}}{\mathbb{E}}\lc\,
                \abs{\,I\lb \mu, \sigma, \alpha\rb - I\lb \Hat{\mu}, \Hat{\sigma}, \alpha\rb\,}
            \,\rc 
        \,=\,            
           \Omega\lb \frac{\alpha^2 \cdot f_{X}\lb \alpha\rb}{n} \rb  
    \end{align*}
\end{conclusion}

\section{Conclusion}
To conclude, in the paper we provided an estimate for the improvement provided by every additional iteration using the DKW inequality, and bounded our error by: 
\begin{align*}
    O\lb \sqrt{\frac{\log k}{k}}\rb   
\end{align*}

\paragraph{}
Next, we bounded from below the best possible mean absolute error of any estimate by
\begin{align*}
   err_{ma} 
   \,\geq\, 
   \frac{\sigma^2}{2 \,\sqrt{2\, \pi}\, k} 
    \cdot 
    \lc\, 1 - F_X\lb \alpha\rb \,\rc 
\end{align*}
and showed an improved bound 
\begin{align*}
    & 
        err_{ma} 
    \,=\,            
       \Omega\lb \frac{\alpha^2 \cdot f_{X}\lb \alpha\rb}{k} \rb  
\end{align*}
for the case $\alpha \geq \mu + \sqrt{\frac{3}{2}} \sigma$. 
This concludes our proof by showing that the bound on our estimator's error is non-trivial and limits the best possible bound we can hope for.

\paragraph{}
We leave for future research to estimate the expected improvement of different methods of hyperparameter tuning such as grid search \cite{claesen2014easy} or bayesian sampling \cite{eggensperger2013towards}.


\newpage
\bibliography{aaai23.bib}

\newpage
\appendix

\section{Appendix A}

\begin{lemma}\label{lem_lb_1} 
    $I$ is monotonic increasing in both 
    $\mu,\, \sigma$, more formally 
    \begin{align*}
        & 
        \frac{\partial}{\partial \mu}\, 
        I\lb \mu, \sigma, \alpha\rb \,\geq\, 0 
        \quad\textit{and}\quad 
        \frac{\partial}{\partial \sigma}\, 
        I\lb \mu, \sigma, \alpha\rb \,\geq\, 0 
    \end{align*}
\end{lemma}
\begin{proof}
    Using $I\,$ definition 
    and denoting 
    $Z \sim \mathcal{N}\lb 0,\, 1\rb$
    we have 
    \begin{align*}
        & 
            I\lb \mu, \sigma, \alpha\rb 
        \,=\, 
            \lb \mu - \alpha\rb
            \cdot
            \lb 1- F_X\lb \alpha \rb \rb
            +
            \sigma^2 \cdot 
            f_X\lb \alpha\rb 
        \\ & 
        \,=\, 
            \lb \mu - \alpha\rb
            \cdot
            \lb 1- F_{\mu + \sigma Z}\lb \alpha \rb \rb
            +
            \sigma^2 \cdot 
            f_{\mu + \sigma Z}\lb \alpha\rb 
        \\ & 
        \,=\, 
            \lb \mu - \alpha\rb
            \cdot
            \lb 1- F_{Z}\lb \frac{\alpha -\mu}{\sigma}\rb \rb
            +
            \sigma^2 \cdot 
            f_{Z}\lb \frac{\alpha -\mu}{\sigma}\rb 
            \cdot \frac{1}{\sigma}
        \\ & 
        \,=\, 
            \lb \mu - \alpha\rb
            \cdot
            \lb 1- F_{Z}\lb \frac{\alpha -\mu}{\sigma}\rb \rb
            +
            \sigma \cdot  
            f_{Z}\lb \frac{\alpha -\mu}{\sigma}\rb 
    \end{align*}
    Denoting 
    \begin{align*}
        &
        I_{1}\lb \mu, \sigma,\alpha\rb 
        \,=\, 
        \lb \mu - \alpha\rb
        \cdot
        \lb 1- F_{Z}\lb \frac{\alpha -\mu}{\sigma}\rb \rb
    \end{align*}
    and 
    \begin{align*}
        &
        I_{2}\lb \mu, \sigma,\alpha\rb 
        \,=\, 
        \sigma \cdot  
        f_{Z}\lb \frac{\alpha -\mu}{\sigma}\rb   
    \end{align*}
    we have 
    \begin{align}\label{eq_4}
        & 
            I\lb \mu, \sigma, \alpha\rb 
        \,=\, 
            I_{1}\lb \mu, \sigma,\alpha\rb 
             \,+\,
            I_{2}\lb \mu, \sigma,\alpha\rb 
    \end{align}
    and thus
    \begin{align*}
        & 
        \frac{\partial}{\partial \mu}\, I\lb \mu, \sigma, \alpha\rb 
        \,=\, 
        \frac{\partial}{\partial \mu}\, 
                I_1\lb \mu, \sigma, \alpha\rb 
        \,+\, 
        \frac{\partial}{\partial \mu}\, 
                I_2\lb \mu, \sigma, \alpha\rb 
    \end{align*}
    and 
    \begin{align*}
        & 
        \frac{\partial}{\partial \sigma}\, I\lb \mu, \sigma, \alpha\rb 
        \,=\, 
        \frac{\partial}{\partial \sigma}\, 
                I_1\lb \mu, \sigma, \alpha\rb 
        \,+\, 
        \frac{\partial}{\partial \sigma}\, 
                I_2\lb \mu, \sigma, \alpha\rb 
    \end{align*}
    computing each term separately we have 
    \begin{align*}
        & 
        \frac{\partial}{\partial \mu}\, 
        I_1\lb \mu, \sigma, \alpha\rb 
            \,\defeq\, 
            \frac{\partial}{\partial \mu}\,
            \lc\,  
                        \lb \mu - \alpha\rb
            \cdot
            \lb 1- F_{Z}\lb \frac{\alpha -\mu}{\sigma}\rb \rb
            \,\rc 
        \\ &  
            \,=\, 
            \underbrace{\lc\, 1- F_{Z}\lb \frac{\alpha -\mu}{\sigma}\rb \,\rc}_{\geq 0} 
            \,-\, 
            \lb \mu - \alpha\rb 
            \frac{\partial}{\partial \mu}\,
            \lc\,  
                F_{Z}\lb \frac{\alpha -\mu}{\sigma}\rb 
            \,\rc 
        \\  & 
            \,\geq\,
            -\lb \mu - \alpha\rb 
            \frac{\partial}{\partial \mu}\,
            \lc\,  
                F_{Z}\lb \frac{\alpha -\mu}{\sigma}\rb 
            \,\rc 
        \\ & 
            \,=\, 
            -\lb \alpha - \mu \rb \cdot 
            f_{Z}\lb \frac{\alpha -\mu}{\sigma}\rb 
            \cdot \frac{1}{\sigma} 
        \\ & 
            \,=\, 
            -\lb \frac{\alpha -\mu}{\sigma}\rb \cdot 
            f_{Z}\lb \frac{\alpha -\mu}{\sigma}\rb 
    \end{align*}
    and 
    \begin{align*}
        & 
        \frac{\partial}{\partial \mu}\, 
        I_2\lb \mu, \sigma, \alpha\rb
        \,=\,
        \sigma \cdot 
         \lb \frac{\alpha -\mu}{\sigma}\rb \cdot    
        f_{Z}\lb \frac{\alpha -\mu}{\sigma}\rb \cdot \frac{1}{\sigma}
        \\ & 
            \,=\,
            \lb \frac{\alpha -\mu}{\sigma}\rb \cdot  
            f_{Z}\lb \frac{\alpha -\mu}{\sigma}\rb
    \end{align*}
    similarly 
    \begin{align*}
        & 
            \frac{\partial}{\partial \sigma}\, 
            I_1\lb \mu, \sigma, \alpha\rb 
        \\ & 
            \,=\,
            \frac{\partial}{\partial \sigma}\,\lc\,  
                        \lb \mu - \alpha\rb
            \cdot
            \lb 1- F_{Z}\lb \frac{\alpha -\mu}{\sigma}\rb \rb \,\rc 
        \\ & 
            \,=\,
            - \lb \mu - \alpha\rb
            \cdot
            \frac{\partial}{\partial \sigma}\,\lc\,  
            F_{Z}\lb \frac{\alpha -\mu}{\sigma}\rb \,\rc 
        \\ & 
            \,=\,
            \lb \mu - \alpha\rb
            \cdot
            \lc\,  
            f_{Z}\lb \frac{\alpha -\mu}{\sigma}\rb 
            \lb \frac{\alpha -\mu}{\sigma^2}\rb\,\rc 
        \\ & 
            \,=\,
            \underbrace{f_{Z}\lb \frac{\alpha -\mu}{\sigma}\rb 
            \lb \frac{\alpha -\mu}{\sigma}\rb^2
             +
             f_{Z}\lb \frac{\alpha -\mu}{\sigma}\rb}_{\geq 0}
        \\ & 
            \,\geq\, 0
    \end{align*}
    and 
    \begin{align*}
        & 
            \frac{\partial}{\partial \sigma}\, 
            I_2\lb \mu, \sigma, \alpha\rb 
        \\ & 
            \,=\,
            \frac{\partial}{\partial \sigma}\,\lc\,  
            \sigma \cdot  
            f_{Z}\lb \frac{\alpha -\mu}{\sigma}\rb \,\rc 
        \\ & 
            \,=\,
            \sigma \cdot  
            \frac{\partial}{\partial \sigma}\, \lc\, 
            f_{Z}\lb \frac{\alpha -\mu}{\sigma}\rb \,\rc  
        \\ & 
            \underbrace{\,=\,}_{Z \sim \mathcal{N}\lb 0, 1\rb}
            \sigma \cdot  
             \lc\, 
            \lb \frac{\alpha -\mu}{\sigma}\rb \cdot 
            f_{Z}\lb \frac{\alpha -\mu}{\sigma}\rb \cdot 
            \lb \frac{\alpha -\mu}{\sigma^2}\rb\,\rc 
        \\ & 
            \underbrace{\,=\,}_{Z \sim \mathcal{N}\lb 0, 1\rb}
            \lb \frac{\alpha -\mu}{\sigma}\rb^2 \cdot 
            f_{Z}\lb \frac{\alpha -\mu}{\sigma}\rb 
            \,\;\geq\;\, 0
    \end{align*}
    combining 
    \begin{align*}
    & 
        \frac{\partial}{\partial \mu}\, 
        I\lb \mu, \sigma, \alpha\rb 
        \,\equiv\,
        \frac{\partial}{\partial \mu}\, 
        I_1\lb \mu, \sigma, \alpha\rb 
         \,+\, 
        \frac{\partial}{\partial \mu}\, 
        I_2\lb \mu, \sigma, \alpha\rb 
    \\ & 
        \,=\, 
         1- F_{Z}\lb \frac{\alpha -\mu}{\sigma}\rb  
            \,-\,\lb \frac{\alpha -\mu}{\sigma}\rb \cdot 
            f_{Z}\lb \frac{\alpha -\mu}{\sigma}\rb
    \\ & \quad\quad\quad\quad\quad\quad\quad\quad\;\;\;
        \,+\, 
         \lb \frac{\alpha -\mu}{\sigma}\rb \cdot 
        f_{Z}\lb \frac{\alpha -\mu}{\sigma}\rb 
    \\ & 
        \,=\, 1- F_{Z}\lb \frac{\alpha -\mu}{\sigma}\rb 
        \,\geq\, 0
    \end{align*}
    and 
    \begin{align*}
    & 
        \frac{\partial}{\partial \sigma}\, 
        I\lb \mu, \sigma, \alpha\rb 
        \,\equiv\,
        \frac{\partial}{\partial \sigma}\, 
        I_1\lb \mu, \sigma, \alpha\rb 
         + 
        \frac{\partial}{\partial \sigma}\, 
        I_2\lb \mu, \sigma, \alpha\rb 
    \\ & 
       \,=\, 
       \lc\, 2 \cdot \lb \frac{\alpha -\mu}{\sigma}\rb^2 + 1 \,\rc \cdot f_{Z}\lb \frac{\alpha -\mu}{\sigma}\rb 
        \,\geq\, 0
    \end{align*}
    as needed $\hfill\square$
\end{proof}

\section{Appendix B}

\begin{lemma}\label{lem_conv_1}
    $I$ is always convex in $\alpha$ 
    it is also convex in $\sigma$ whenever 
    \begin{align*}
         \alpha   
        \;\leq\; \mu - \sqrt{\frac{3}{2}}\sigma 
        \;\;\;\textit{or}\;\;\; 
        \alpha   
        \;\geq\; \mu + \sqrt{\frac{3}{2}}\sigma\,   
    \end{align*}
    more formally 
    \begin{align*}
        & 
        \frac{\partial^2}{\partial^2 \mu}\, 
        I\lb \mu, \sigma, \alpha\rb \,\geq\, 0 
        \quad\textit{and}\quad
        \frac{\partial^2}{\partial^2 \sigma}\, 
        I\lb \mu, \sigma, \alpha\rb \,\geq\, 0 
    \end{align*}
    when the second inequality holds in the specified range. 
\end{lemma}
\begin{proof}
    Taking  
    $\frac{\partial}{\partial \mu}\, I\lb \mu, \sigma, \alpha\rb,\, 
    \frac{\partial}{\partial \sigma}\, I\lb \mu, \sigma, \alpha\rb$ 
    formulas from proposition (\ref{prop_1}) we have 
    \begin{align*}
        &
            \frac{\partial^2}{\partial^2 \mu}\, I\lb \mu, \sigma, \alpha\rb 
        \,=\,
            \frac{\partial}{\partial \mu}\, \lb 1- F_{Z}\lb \frac{\alpha -\mu}{\sigma}\rb \rb  
        \\ & 
        \,=\,
             f_{Z}\lb \frac{\alpha -\mu}{\sigma}\rb 
             \cdot \frac{1}{\sigma} 
        \,\geq\, 0
    \end{align*}
    similarly 
    \begin{align*}
        &
            \frac{\partial^2}{\partial^2 \sigma}\, I\lb \mu, \sigma, \alpha\rb 
        \\ & 
        \,=\,
            \frac{\partial}{\partial \sigma}\; \lc\, 2 \cdot \lb \frac{\alpha -\mu}{\sigma}\rb^2 + 1 \,\rc 
            \cdot f_{Z}\lb \frac{\alpha -\mu}{\sigma}\rb 
        \\ & 
        \,=\,  
            - 4 \cdot 
            \frac{1}{\sigma} \cdot \lb \frac{\alpha -\mu}{\sigma}\rb^2 
            \cdot f_{Z}\lb \frac{\alpha -\mu}{\sigma}\rb
        \\ & 
            + \lc\, 2 \cdot \lb \frac{\alpha -\mu}{\sigma}\rb^2 + 1 \,\rc 
            \cdot f_{Z}\lb \frac{\alpha -\mu}{\sigma}\rb \cdot \lb \frac{\alpha -\mu}{\sigma}\rb^2 
            \cdot \frac{1}{\sigma} 
        \\ & 
            \,=\, 
            \lc\, 2 \cdot \lb \frac{\alpha -\mu}{\sigma}\rb^2 - 3 \,\rc 
            \cdot f_{Z}\lb \frac{\alpha -\mu}{\sigma}\rb \cdot \lb \frac{\alpha -\mu}{\sigma}\rb^2 
            \cdot \frac{1}{\sigma} 
    \end{align*}
    thus 
    \begin{align*}
        &
            \frac{\partial^2}{\partial^2 \sigma}\, I\lb \mu, \sigma, \alpha\rb \,\geq\,0
    \end{align*}
    \begin{align*}
            \Updownarrow 
    \end{align*}
    \begin{align*}
        & 
            \lc\, 2 \cdot \lb \frac{\alpha -\mu}{\sigma}\rb^2 - 3 \,\rc 
            \cdot f_{Z}\lb \frac{\alpha -\mu}{\sigma}\rb \cdot \lb \frac{\alpha -\mu}{\sigma}\rb^2 
            \cdot \frac{1}{\sigma} \,\geq\, 0
    \end{align*}
    \begin{align*}
            \Updownarrow 
    \end{align*}
    \begin{align*}
        & 
             2 \cdot \lb \frac{\alpha -\mu}{\sigma}\rb^2 - 3  
             \,\geq\, 0
    \end{align*}
    \begin{align*}
            \Updownarrow 
    \end{align*}
    \begin{align*}
        & 
        \,\frac{\alpha -\mu}{\sigma}\, 
        \;\leq\; -\sqrt{\frac{3}{2}}         
        \quad\textit{or}\quad
        \,\frac{\alpha -\mu}{\sigma}\, 
        \;\geq\; \sqrt{\frac{3}{2}} 
    \end{align*}
    \begin{align*}
            \Updownarrow 
    \end{align*}
    \begin{align*}
        & 
            \alpha   
            \;\leq\; \mu - \sqrt{\frac{3}{2}}\sigma 
            \quad\textit{or}\quad 
            \alpha   
            \;\geq\; \mu + \sqrt{\frac{3}{2}}\sigma 
    \end{align*}
    as needed $\hfill\square$
\end{proof}

\section{Appendix C}

\begin{lemma}\label{lem_23}
    The following inequality holds 
    \begin{align*}
        \forall s,\, x,\, y \in \mathbb{R}_{>0}\quad\quad
        \frac{\Gamma\lb s, x+y\rb}{\Gamma\lb s, x\rb} 
        \,\geq\, e^{-y} 
    \end{align*}
\end{lemma}
\begin{proof}
    Starting from $\Gamma$ definition 
    \begin{align*}
        & 
        \Gamma\lb s, x+y\rb 
        \,=\,  
        \int_{x+y}^\infty\, t^{s-1} e^{-t} \, d t 
        \\ & 
        \,=\,  
        \int_{x}^\infty\, \lb t + y\rb^{s-1} e^{-t-y} \, d t
        \\ & 
        \,=\,  
        e^{-y} \int_{x}^\infty\, \lb t + y\rb^{s-1} e^{-t} \, d t
        \\ & 
        \,\geq\, 
        e^{-y} \int_{x}^\infty\, t^{s-1} e^{-t} \, d t
        \\ & 
        \,=\,  
        e^{-y}\cdot \Gamma\lb s, x\rb 
    \end{align*}
    reordering yields 
    \begin{align*}
        \frac{\Gamma\lb s, x+y\rb}{\Gamma\lb s, x\rb} 
        \,\geq\, e^{-y} 
    \end{align*}
    as needed $\hfill\square$
\end{proof}

\end{document}